\documentclass[11pt]{article}
\usepackage[dvips]{graphicx}
\usepackage{amssymb,amsmath,color}
\usepackage{url}

\newenvironment{itemizeReduced}{
\begin{list}{\labelitemi}{\leftmargin=1.8em}
\setlength{\itemsep}{3pt}
\setlength{\parskip}{0pt}
\setlength{\parsep}{0pt}}{\end{list}
}

\newcommand{\BEAS}{\begin{eqnarray*}}
\newcommand{\EEAS}{\end{eqnarray*}}
\newcommand{\BEA}{\begin{eqnarray}}
\newcommand{\EEA}{\end{eqnarray}}
\newcommand{\BEQ}{\begin{equation}}
\newcommand{\EEQ}{\end{equation}}
\newcommand{\BIT}{\begin{itemizeReduced}}
\newcommand{\EIT}{\end{itemizeReduced}}
\newcommand{\BNUM}{\begin{enumerate}}
\newcommand{\ENUM}{\end{enumerate}}
\newcommand{\BA}{\begin{array}}
\newcommand{\EA}{\end{array}}

\newcommand{\rb}{\mathbb{R}}
\newcommand{\BlackBox}{\rule{1.5ex}{1.5ex}}  

\newcommand{\mysec}[1]{Section~\ref{sec:#1}}
\newcommand{\eq}[1]{Eq.~(\ref{eq:#1})}

\newenvironment{proof}{\par\noindent{\bf Proof\ }}{\hfill\BlackBox\\[2mm]}
\newtheorem{lemma}{Lemma}

\newtheorem{proposition}{Proposition}

\title{Duality between subgradient and conditional gradient methods}

\author{
Francis Bach  \\
INRIA - Sierra project-team\\
D\'epartement d'Informatique de l'Ecole Normale Sup\'erieure \\
Paris, France \\
\texttt{francis.bach@ens.fr} }

\begin{document}

\maketitle

 \begin{abstract}
 Given a convex optimization problem and its dual, there are many possible first-order algorithms. In this paper, we show the equivalence between mirror descent algorithms and algorithms generalizing the conditional gradient method. This is done through convex duality, and implies notably that for certain problems, such as for supervised machine learning problems with non-smooth losses or problems regularized by non-smooth regularizers, the primal subgradient method and the dual conditional gradient method are formally equivalent. The dual interpretation leads to a form of line search for mirror descent, as well as guarantees of convergence for primal-dual certificates. \end{abstract}

 \section{Introduction}
 Many problems in machine learning, statistics and signal processing may be cast as convex optimization problems. In large-scale situations, simple gradient-based algorithms with potentially many cheap iterations are often preferred over methods, such as Newton's method or interior-point methods, that rely on fewer but more expensive iterations. The choice of a first-order method depends on the structure of the problem, in particular (a) the smoothness and/or strong convexity of the objective function, and (b) the computational efficiency of certain operations related to the non-smooth parts of the objective function, when it is decomposable in a smooth and a non-smooth part.
 
 In this paper, we consider two classical algorithms, namely (a) subgradient descent and its mirror descent extension \cite{shor1985minimization,nemirovsky1983problem,beck2003mirror}, and (b) conditional gradient algorithms, sometimes referred to as Frank-Wolfe algorithms~\cite{frank2006algorithm,dem1967minimization,dunn1978conditional,dunn1980convergence,jaggi}. 
 
 Subgradient algorithms are adapted to non-smooth unstructured situations, and after $t$ steps have a convergence rate of $O(1/\sqrt{t})$ in terms of objective values. This convergence rate improves to $O(1/t)$ when the objective function is strongly convex~\cite{nedic2000convergence}.
Conditional-gradient algorithms are tailored to the optimization of smooth functions on a compact convex set, for which minimizing linear functions is easy (but where orthogonal projections would be hard, so that proximal methods \cite{Nesterov2007,Beck2009}  cannot be used efficiently).
They also have a convergence rate of $O(1/t)$~\cite{dunn1978conditional}. The main results of this paper are (a) to show that for common situations in practice, these two sets of methods are in fact equivalent by convex duality,   (b) to recover a previously proposed  extension of the conditional gradient method which is more generally applicable~\cite{bredies2008iterated}, and (c) provide explicit convergence rates for primal and dual iterates. We also review in Appendix~\ref{app:nonstrongly} the non-strongly convex case and show that both primal and dual suboptimalities then converge at rate $O(1/\sqrt{t})$.

 More precisely, we consider a convex function $f$ defined on $\rb^n$, a  convex function $h$ defined on $\rb^p$, both potentially taking the value $+\infty$,  and a matrix  $ A \in \rb^{n \times p}$.  We consider the following minimization problem, which we refer to as the \emph{primal} problem:
 \BEQ
 \label{eq:primal}
 \min_{x \in \rb^p}   \  h(x) + f(Ax ).
 \EEQ
  Throughout this paper, we make the following assumptions regarding the problem:
\BIT
  \item[--] \emph{$f$ is Lipschitz-continuous and finite on $\rb^n$}, i.e., there exists a constant $B$ such that
   for all $x,y \in \rb^n$, $|f(x) - f(y)| \leqslant B \| x- y \|$, where $\| \cdot \|$ denotes the Euclidean norm. Note that this implies that the domain of  the Fenchel conjugate $f^\ast$  is bounded. We denote by $C$ the bounded domain of $f^\ast$. Thus, for all $z \in \rb^n$, $f(z) = \max_{ y \in C} y^\top z - f^\ast(y)$. In many situations, $C$ is also closed but this is not always the case (in particular, when $f^\ast(y)$ tends to infinity when $y$ tends to the boundary of $C$).
   
   Note that the boundedness of the domain of $f^\ast$ is crucial and allows for simpler proof techniques with explicit constants (see a generalization in~\cite{bredies2008iterated}).

   \item[--] \emph{$h$ is lower-semicontinuous and $\mu$-strongly convex on $\rb^p$}. This implies that $h^\ast$ is defined on $\rb^p$, differentiable with $(1/\mu)$-Lipschitz continuous gradient~\cite{borwein2006caa,rockafellar97}. Note that the domain $K$ of $h$ may be strictly included in $\rb^p$.

   \EIT

    \vspace*{.25cm}
    
    Moreover, we assume that the following quantities may be computed efficiently:
   \BIT
    \item[--] \emph{Subgradient of $f$}: for any $z \in \rb^n$, a subgradient of $f$ is any maximizer $y$ of
  $ \max_{y \in C} y^\top z - f^\ast(y)$.
    
    \item[--] \emph{Gradient of $h^\ast$}: for any $z \in \rb^p$, $(h^\ast)'(z)$ may be computed and is equal to the unique maximizer $x$ of  $ \max_{x \in \rb^p} x^\top z - h(x)$.
     \EIT

 \vspace*{.25cm}
 
  The values of the functions $f$, $h$, $f^\ast$ and $h^\ast$ will be useful       to compute duality gaps but are not needed to run the algorithms. As shown in \mysec{examples}, there are many examples of pairs of functions with the computational constraints described above. If other operations are possible, in particular $\max_{y \in C} y^\top z - f^\ast(y) - \frac{\varepsilon}{2}\|y\|^2$, then proximal methods~\cite{Beck2009,Nesterov2007} applied to the dual problem converge at rate $O(1/t^2)$. If $f$ and $h$ are smooth, then gradient methods (accelerated~\cite[Section 2.2]{Nesterov2004} or not) have linear convergence rates.

We denote by $g_{\rm primal}(x) =  h(x) + f(Ax )$ the primal objective in \eq{primal}.
It is the sum of a Lipschitz-continuous convex function and a strongly convex function, potentially on a restricted domain $K$. It is thus well adapted to the subgradient method~\cite{shor1985minimization}.

We have the following primal/dual relationships (obtained from Fenchel duality~\cite{borwein2006caa}):
\BEAS
 \min_{x \in \rb^p}   h(x) + f(Ax )
 & = &   \min_{x \in \rb^p}  \max_{ y \in C}    h(x)  + y^\top ( A x  ) - f^\ast(y)
\\
 & = &   \max_{ y \in C}   \bigg\{ \min_{x \in \rb^p}     h(x) + x^\top A^\top y 
 \bigg\} - f^\ast(y)  \\
 & = &   \max_{ y \in C}  -    h^\ast(-A^\top y  )  - f^\ast(y).
\EEAS

This leads to the \emph{dual} maximization problem:
\BEQ
\label{eq:dual}
\max_{ y \in C}     -    h^\ast(-A^\top y )  - f^\ast(y).
\EEQ
 We denote by $ g_{\rm dual}(y) =  -    h^\ast(-A^\top y )  - f^\ast(y)
$ the dual objective. It has a smooth part   $-    h^\ast(-A^\top y ) $ defined on $\rb^n$ and a potentially non-smooth part $ - f^\ast(y)$, and the problem is restricted onto a \emph{bounded} set $C$. When $f^\ast$ is linear (and more generally smooth) on its support, then we are exactly in the situation where conditional gradient algorithms may be used~\cite{frank2006algorithm,dem1967minimization}.

Given a pair of primal-dual candidates $(x,y)  \in K \times C$, we denote by ${\rm gap}(x,y)$ the duality gap:
$$
{\rm gap}(x,y) = g_{\rm primal}(x)  -  g_{\rm dual}(y)
=  \big[ h(x) + h^\ast(-A^\top y ) +    y^\top Ax \big ]   + \big[  f(Ax) + f^\ast(y) - y^\top Ax 
\big] .$$
It is equal to zero if and only if (a) $(x,-A^\top y)$ is a Fenchel-dual pair for $h$ and (b) $(Ax,y)$ is a Fenchel-dual pair for $f$.
This quantity serves as a certificate of optimality, as 
$$ 
{\rm gap}(x,y) = \big[ g_{\rm primal}(x) - \min_{x' \in K} g_{\rm primal}(x') \big]  
+\big[ \max_{y' \in C} g_{\rm dual}(y') -  g_{\rm dual}(y) \big].
$$
  The goal of this paper is to show that for certain problems ($f^\ast$ linear and $h$ quadratic), the subgradient method applied to the primal problem in \eq{primal} is equivalent to the conditional gradient applied to the dual problem in \eq{dual}; when relaxing the assumptions above, this equivalence is then between mirror descent methods and generalized conditional gradient algorithms.

 \section{Examples}
 \label{sec:examples}
 The non-smooth strongly convex optimization problem defined in \eq{primal} occurs in many applications in machine  learning and signal processing, either because they are formulated directly in this format, or their dual in \eq{dual} is (i.e., the original problem is the minimization of a smooth function over a compact set).
 
 \subsection{Direct formulations}
 
 Typical cases for $h$ (often the regularizer in machine learning and signal processing) are   the following:
 \BIT
 \item[--]\emph{Squared Euclidean norm}: $h(x) = \frac{\mu}{2} \| x\|^2$, which is $\mu$-strongly convex.
 
 \item[--]\emph{Squared Euclidean norm with convex constraints}: $h(x) = \frac{\mu}{2} \| x\|^2 + I_K(x)$, with $I_K$ the indicator function for $K$ a closed  convex set, which is $\mu$-strongly convex.
 
 \item[--]\emph{Negative entropy}: $h(x) = \sum_{i=1}^n x_i \log x_i + I_K(x)$, where $K = \{ x \in \rb^n,\ x \geqslant 0, \ \sum_{i=1}^n x_i=1 \}$, which is $1$-strongly convex. More generally, many barrier functions of convex sets may be used (see examples in \cite{beck2003mirror,boyd.convex}, in particular for problems on matrices). 
  \EIT
 
 \vspace*{.25cm}

 Typical cases for $f$ (often the data fitting terms in machine learning and signal processing) are functions of the form $f(z) = \frac{1}{n}\sum_{i=1}^n \ell_i(z_i)$:
 \BIT
 \item[--]\emph{Least-absolute-deviation}: $\ell_i(z_i) = | z_i - y_i| $, with $y_i \in \rb$. Note that the square loss is not Lipschitz-continuous on $\rb$ (although it is Lipschitz-continuous when restricted to a bounded set).

  \item[--]\emph{Logistic regression}: $\ell_i(z_i) = \log( 1+ \exp( - z_i y_i) )$, with $y_i \in \{-1,1\}$. Here $f^\ast$ is not linear in its support, and $f^\ast$ is not smooth, since it is a sum of negative entropies (and the second-order derivative is not bounded). This extends to any ``log-sum-exp'' functions which occur as a negative  log-likelihood from the exponential family~(see, e.g., \cite{wainwright2008graphical} and references therein). 
  Note that $f$ is then smooth and   proximal methods with an exponential convergence rate may be used (which correspond to a constant step size in the algorithms presented below, instead of a decaying step size)~\cite{Nesterov2007,Beck2009}.
  
   \item[--]\emph{Support vector machine}: $\ell_i(z_i) = \max\{1-y_i z_i,0\}$, with $y_i \in \{-1,1\}$. Here  $f^\ast$ is linear on its domain (this is  a situation where subgradient and conditional gradient methods are exactly equivalent). This extends to more general ``max-margin'' formulations~\cite{tsochantaridis2006large,taskar2005learning}: in these situations, a combinatorial object (such as a full chain, a graph, a matching or vertices of the hypercube) is estimated (rather than an element of $\{-1,1\}$)  and this leads to functions $z_i \mapsto \ell_i(z_i)$ whose Fenchel-conjugates are linear and have domains which are related to the polytopes associated to the linear programming relaxations of the corresponding combinatorial optimization problems. For these polytopes, often, only linear functions can be maximized, i.e., we can compute a subgradient of $\ell_i$ but typically nothing more.
   
 \EIT

 \vspace*{.25cm}
 
 Other examples may be found in signal processing; for example, total-variation denoising, where the loss is strongly convex but the regularizer is non-smooth~\cite{chambolle2011first}, or submodular function minimization cast through separable optimization problems~\cite{bach2011learning}. Moreover, many proximal operators for non-smooth regularizers are of this form, with $h(x) = \frac{1}{2} \| x - x_0\|^2$ and $f$ is a norm (or more generally a gauge function).

   \subsection{Dual formulations}
Another interesting set of  examples for machine learning are more naturally described from the dual formulation in \eq{dual}: given a smooth loss term $  h^\ast(-A^\top y)$ (this could be least-squares or logistic regression), a typically non-smooth penalization or constraint is added, often through a norm $\Omega$. Thus, this corresponds to functions $f^\ast$ of the form $f^\ast(y) = \varphi( \Omega(y) )$, where $\varphi$ is a convex non-decreasing function ($f^\ast$ is then convex). 

Our main assumption is that a subgradient of $f$ may be easily computed. This is equivalent to being able to maximize functions of the form $z^\top y - f^\ast(y) = z^\top y  - \varphi(\Omega(y))$ for $z \in \rb^n$. If one can compute the dual norm of $z$, $\Omega^\ast(z) = \max_{ \Omega(y) \leqslant 1} z^\top y$, and in particular a maximizer $y$ in the unit-ball of $\Omega$, then one can compute simply the subgradient of $f$. Only being able to compute the dual norm efficiently is a common situation in machine learning and signal processing, for example, for structured regularizers based on submodularity~\cite{bach2011learning}, all atomic norms~\cite{chandrasekaran2012convex}, and norms based on matrix decompositions~\cite{siammatrix}. See additional examples in~\cite{jaggi}.

Our assumption regarding the compact domain of $f^\ast$ translates to the assumption that $\varphi$ has compact domain. This includes indicator functions $\varphi = I_{[0, \omega_0]}$ which corresponds to the constraint $\Omega(y) \leqslant \omega_0$. We may also consider $\varphi(\omega) = \lambda \omega +  I_{[0, \omega_0]}(\omega)$, which corresponds to jointly penalizing and constraining the norm; in practice, $\omega_0$ may be chosen so that the constraint $\Omega(y) \leqslant \omega_0$ is not active at the optimum and we get the solution of the penalized problem $\max_{y \in \rb^n} -h^\ast(-A^\top y) - \lambda \Omega(y)$. See~\cite{zaidcg,zhang2012accelerated,siammatrix} for alternative approaches.

\section{Mirror descent for strongly convex problems}
\label{sec:mirror}
We first assume that the function $h$ is \emph{essentially smooth} (i.e., differentiable at any point in the interior of $K$, and so that the norm of gradients converges to $+\infty$ when approaching the boundary of $K$); then  $h'$ is a bijection from ${\rm int}(K)$ to $\rb^p$, where $K$ is the domain of $h$ (see, e.g.,~\cite{rockafellar97,hiriart1996convex}). We consider the  Bregman divergence
$$D(x_1,x_2) = h(x_1) - h(x_2) - ( x_1 - x_2)^\top h'(x_2). $$
 It is always defined on $K \times {\rm int}(K)$, and is nonnegative. 
If $x_1,x_2 \in {\rm int}(K)$, then $D(x_1,x_2)=0$ if  and only if $x_1=x_2$. Moreover, since $h$ is assumed $\mu$-strongly convex, we have $D(x_1,x_2) \geqslant \frac{\mu}{2} \| x_1 - x_2\|^2$. See more details in~\cite{beck2003mirror}.
 For example, when $h(x) = \frac{\mu}{2} \| x\|^2$, we have $D(x_1,x_2) = \frac{\mu}{2} \| x_1 - x_2\|^2$.

\paragraph{Subgradient descent for square Bregman divergence}
We first consider the common situation where $h(x) = \frac{\mu}{2} \| x\|^2 $; the primal problem then becomes:
$$
\min_{x \in K}   f(Ax ) + \frac{\mu}{2} \| x\|^2.
$$

The projected subgradient method starts from any $x_0 \in \rb^p$, and iterates the following recursion:
$$
x_t = x_{t-1} - \frac{\rho_t}{\mu} \big[
A^\top f'(A x_{t-1}) + \mu x_{t-1}
\big]  ,
$$
where $\bar{y}_{t-1} = f'(A x_{t-1})$ is any subgradient of $f$ at $A x_{t-1}$. The step size is $ \displaystyle \frac{\rho_t}{\mu} $.

The recursion may be rewritten as
$$
\mu x_t =  \mu x_{t-1} -  {\rho_t} \big[
A^\top f'(A x_{t-1}) + \mu x_{t-1}
\big] ,
$$
which is equivalent to $x_t$ being the unique minimizer of
\BEQ
\label{eq:AA}
 (  x - x_{t-1})^\top  \big[ A^\top \bar{y}_{t-1} +   \mu x_{t-1} \big] +  \frac{\mu}{2 \rho_t} \| x - x_{t-1}\|^2  ,
\EEQ
which is the traditional proximal step, with step size $\rho_t /\mu$.


\paragraph{Mirror descent}
We may interpret the last formulation in \eq{AA} for the square regularizer $h(x) = \frac{\mu}{2} \| x\|^2$ as the minimization
of
    $$  (  x - x_{t-1})^\top  g'_{\rm primal}(x_{t-1}) +  \frac{1}{ \rho_t} D(x,x_{t-1}),
$$
with solution defined through (note that $h'$ is a bijection from ${\rm int}(K)$ to $\rb^p$):
\BEAS
h'(x_t)  & = &  h'(x_{t-1}) -    \rho_t \big[
A^\top f'(A x_{t-1}) + h'(x_{t-1})  \big]  \\
& = & ( 1 - \rho_t)  h'(x_{t-1})  -    \rho_t  
A^\top f'(A x_{t-1}).
\EEAS
 This leads to the following definition of the mirror descent recursion:
\BEQ
\label{eq:mirror}
\left\{
\begin{array}{ccl} 
  \bar{y}_{t-1} &  \in &\displaystyle \arg \max_{ y \in C } \ y^\top A x_{t-1} - f^\ast(y) , \\[.25cm]
    x_t & = & \displaystyle \arg \min_{x \in \rb^p } \ h(x) -  ( 1 -   \rho_t ) x^\top h'(x_{t-1}) + \rho_t  x^\top A^\top \bar{y}_{t-1} .
    \end{array} \right.
\EEQ
 
 The following proposition proves the convergence of mirror descent in the strongly convex case with rate $O(1/t)$---previous results were considering the convex case, with convergence rate $O(1/\sqrt{t})$~\cite{nemirovsky1983problem,beck2003mirror}.
 \begin{proposition}[Convergence of mirror descent in the strongly convex case]
  \label{prop:mirror}
  Assume that (a) $f$ is Lipschitz-continuous and finite on $\rb^p$, with $C$ the domain of~$f^\ast$, (b) $h$ is essentially smooth and $\mu$-strongly convex. Consider $\rho_t = 2/(t+1)$ and $R^2 = \max_{y,y' \in C} \| A^\top ( y - y')\|^2$. Denoting by $x_\ast$ the unique minimizer of $g_{\rm primal}$,  after~$t$ iterations of the mirror descent recursion of \eq{mirror}, we have:
 \BEAS
 g \bigg( \frac{2}{t(t+1)}\sum_{u=1}^t u  x_{u-1}  \bigg)
- g_{\rm primal} (x_\ast)  
  &  \leqslant  & 
 \frac{R^2   }{  \mu (t+1)  } , \\
 \min_{ u \in \{0,\dots,t-1\}}  \Big\{ g_{\rm primal}(x_u) -  g_{\rm primal} (x_\ast) \Big\}  & \leqslant  &\frac{R^2   }{  \mu (t+1)  }, \\
  D ( x_\ast, x_{t}) & \leqslant  &  \frac{R^2   }{  \mu (t+1)  }.
 \EEAS
 \end{proposition}
 
 \begin{proof}
We follow the proof of~\cite{beck2003mirror} and adapt it to the strongly convex case. We have, by reordering terms and using the optimality condition $h'(x_t) =  h'(x_{t-1}) -    \rho_t \big[
A^\top f'(A x_{t-1}) + h'(x_{t-1})  \big]  $:
\BEA
\nonumber \ \ \ \ & & D( x_\ast, x_t) - D( x_\ast, x_{t-1}) \\
 \nonumber & = & h(x_{t-1}) - h(x_t) - ( x_\ast - x_{t})^\top h'(x_{t})
+ ( x_\ast - x_{t-1})^\top h'(x_{t-1})
\\
\nonumber& = &  h(x_{t-1}) - h(x_t) - ( x_\ast - x_{t})^\top
\big[ ( 1 - \rho_t)  h'(x_{t-1})  -    \rho_t  
A^\top f'(A x_{t-1}) \big] \\
\nonumber & & \hspace*{7.25cm}
+ ( x_\ast - x_{t-1})^\top h'(x_{t-1})\\
\nonumber& = & 
 h(x_{t-1}) - h(x_t) - ( x_{t-1}- x_{t})^\top h'(x_{t-1})
 + \rho_t (x_\ast - x_{t} ) ^\top g'_{\rm primal}(x_{t-1})
\\
\label{eq:A} & = & \big[ -D(x_t,x_{t-1})  
 +\rho_t (x_{t-1} - x_{t} ) ^\top g'_{\rm primal}(x_{t-1}) \big]
 \\
\nonumber & & \hspace*{6cm}
 + \big[ \rho_t (x_{\ast} - x_{t-1} ) ^\top g'_{\rm primal}(x_{t-1}) \big]
.
\EEA
In order to upper-bound the two terms in \eq{A}, we first  consider the following bound (obtained by convexity of $f$ and the definition of $D$):
$$f(Ax_\ast) + h(x_\ast) \geqslant  f(Ax_{t-1}) + h(x_{t-1}) + 
( x_\ast - x_{t-1})^\top  [  A^\top \bar{y}_{t-1} +  h'(x_{t-1}) ] + 
 D ( x_\ast, x_{t-1}),
$$
which may be rewritten as:
$$g_{\rm primal}(x_{t-1}) - g_{\rm primal} (x_\ast) \leqslant  -  D ( x_\ast, x_{t-1}) + 
( x_{t-1} - x_{\ast})^\top g'_{\rm primal}(x_{t-1}),
$$
which implies 
\BEQ
\label{eq:B}
  \ \ \rho_t (x_{\ast} - x_{t-1} ) ^\top g'_{\rm primal}(x_{t-1}) \leqslant 
 -  \rho_t D ( x_\ast, x_{t-1}) - \rho_t  \big[
 g_{\rm primal}(x_{t-1}) - g_{\rm primal} (x_\ast) 
 \big].
\EEQ
Moreover, by definition of $x_t$,
 $$ - D(x_t,x_{t-1})  
 + \rho_t (x_{t-1} - x_{t} ) ^\top g'_{\rm primal}(x_{t-1}) = \max_{x \in \rb^p}
 - D(x,x_{t-1})  
 +\rho_t (x_{t-1} - x ) ^\top z = \varphi(z),
$$
with $z  = \rho_t g'_{\rm primal}(x_{t-1}) $. The function $x \mapsto D(x,x_{t-1})  $ is $\mu$-strongly convex, and its Fenchel conjugate is thus $(1/\mu)$-smooth. This implies that $\varphi$ is $(1/\mu)$-smooth. Since $\varphi(0) = 0$ and $\varphi'(0)=0$, $\varphi(z) \leqslant \frac{1}{2\mu} \| z\|^2$. Moreover, 
$z = \rho_t \big[ A^\top f'(Ax_{t-1}) + h(x_{t-1})\big]$. Since $h'(x_{t-1}) \in -A^\top C$ (because 
$h'(x_{t-1})$ is a convex combination of such elements), then
$\| A^\top f'(Ax_{t-1}) + h(x_{t-1})\|^2 \leqslant R^2 = \max_{ y_1,y_2 \in C} \| A^\top ( y_1 - y_2)\|^2 
= {\rm diam}( A^\top C)^2$.

Overall, combining \eq{B} and $\varphi(z) \leqslant \frac{R^2 \rho_t^2}{2\mu}   $ into \eq{A}, this implies that 
$$
  D( x_\ast, x_t) - D( x_\ast, x_{t-1})  
  \leqslant   
\frac{\rho_t^2}{2\mu}R^2 
 -  \rho_t D ( x_\ast, x_{t-1}) - \rho_t  \big[
 g_{\rm primal}(x_{t-1}) - g_{\rm primal} (x_\ast)  \big],
$$
that is, 
$$
 g_{\rm primal}(x_{t-1}) - g_{\rm primal} (x_\ast)
 \leqslant 
 \frac{ \rho_t R^2}{2 \mu} + ( \rho_t^{-1} - 1) D ( x_\ast, x_{t-1})  - \rho_t^{-1} D ( x_\ast, x_{t}) .
$$
With $\displaystyle \rho_t = \frac{2}{t+1}$, we obtain
\BEAS
 t \big[ g_{\rm primal}(x_{t-1}) - g_{\rm primal} (x_\ast) \big]
 & \leqslant & 
 \frac{R^2 t   }{  \mu (t+1)  } +   \frac{(t-1)t }{2}  D ( x_\ast, x_{t-1})  - \frac{t(t+1)}{2} D ( x_\ast, x_{t}).
\EEAS
Thus, by summing from $u=1$ to $u=t$,  we obtain
 $$
 \sum_{u=1}^t u \big[ g_{\rm primal}(x_{u-1}) - g_{\rm primal} (x_\ast) \big]
   \leqslant 
 \frac{R^2   }{  \mu  } t   - \frac{t(t+1)}{2} D ( x_\ast, x_{t}),
$$
that is,
$$
 D ( x_\ast, x_{t}) +  \frac{2}{t(t+1)}\sum_{u=1}^t u \big[ g_{\rm primal}(x_{u-1}) - g_{\rm primal} (x_\ast) \big]
   \leqslant 
 \frac{R^2   }{  \mu (t+1)  }   .  
$$
This implies that  $ \displaystyle D ( x_\ast, x_{t}) \leqslant  \frac{R^2   }{  \mu (t+1)  }    $, i.e., the iterates converges. Moreover,
using the convexity of $g$,  
$$
g \bigg( \frac{2}{t(t+1)}\sum_{u=1}^t u  x_{u-1}  \bigg)
- g_{\rm primal} (x_\ast)  
   \leqslant 
    \frac{2}{t(t+1)}\sum_{u=1}^t u \big[ g_{\rm primal}(x_{u-1}) - g_{\rm primal} (x_\ast) \big]
   \leqslant 
 \frac{R^2   }{  \mu (t+1)  },
$$
i.e., 
the objective functions at an averaged iterate converges, and  
$$
\min_{ u \in \{0,\dots,t-1\}} g_{\rm primal}(x_u) -  g_{\rm primal} (x_\ast)   \leqslant  \frac{R^2   }{  \mu (t+1)  },
$$
i.e., one of the iterates has an objective that converges.
\end{proof}

\paragraph{Averaging}

Note that with the step size $\displaystyle \rho_t = \frac{2}{t+1}$, we have
$$h'(x_t) = \frac{ t-1}{t+1} h'(x_{t-1}) - \frac{2}{t+1} A^\top f'(A x_{t-1}),$$ which implies
$$t (t+1) h'(x_t) =  (t-1) t h'(x_{t-1}) - 2 t A^\top f'(A x_{t-1}).$$
By summing these equalities, we obtain
$t (t+1) h'(x_t) = - 2 \sum_{u=1}^t   u A^\top f'(A x_{u-1})$, i.e.,
$$h'(x_t) = \frac{2}{t (t+1) } \sum_{u=1}^t   u  \big[ -   A^\top f'(A x_{u-1}) \big],$$
that is, $h'(x_t)$ is a weighted average of subgradients (with more weights on later iterates).

For $\rho_t = 1 /t$, then, we the same techniques, we would obtain a convergence rate proportional to $\frac{R^2}{\mu t} \log t$ for the average iterate $\frac{1}{t} \sum_{u=1}^t x_{u-1}$, thus with an additional $\log t$ factor (see a similar situation in the stochastic case in~\cite{lacoste2012stochastic}). We would then have $h'(x_t) = 
 \frac{1}{t   } \sum_{u=1}^t     \big[ -   A^\top f'(A x_{u-1}) \big]$, and this is exactly a form dual averaging method~\cite{nesterov2009primal}, which also comes with primal-dual guarantees.

\paragraph{Generalization to $h$ non-smooth}
The previous result does not require $h$ to be essentially smooth, i.e., it may be applied to $h(x) = \frac{\mu}{2} \| x\|^2 + I_K(x)$ where $K$ is a closed convex set strictly included in $\rb^p$. In the mirror descent recursion,
$$
\left\{
\begin{array}{ccl} 
  \bar{y}_{t-1} &  \in &\displaystyle \arg \max_{ y \in C } \ y^\top A x_{t-1} - f^\ast(y) , \\[.25cm]
    x_t & = & \displaystyle \arg \min_{x \in \rb^p } \ h(x) -  ( 1 -   \rho_t ) x^\top h'(x_{t-1}) + \rho_t  x^\top A^\top \bar{y}_{t-1} ,
    \end{array} \right. 
$$
there may then be multiple choices for $h'(x_{t-1})$. If we choose for $h'(x_{t-1})$ at iteration~$t$, the subgradient of $h$ obtained at the previous iteration, i.e., such
that $h'(x_{t-1}) =    ( 1 -   \rho_{t-1} )   h'(x_{t-2}) - \rho_{t-1}    A^\top \bar{y}_{t-2}$, then the proof of Prop.~\ref{prop:mirror} above holds.

Note that when $h(x) = \frac{\mu}{2} \| x\|^2 + I_K(x)$, the algorithm above is \emph{not} equivalent to classical projected gradient descent. Indeed, the classical algorithm has the iteration
$$
x_{t} = \Pi_K \bigg( \! x_{t-1} - \frac{1}{\mu} \rho_t \big[ \mu x_{t-1} + A^\top f'(A x_{t-1}) \big]\!\bigg)= \Pi_K \bigg(\! ( 1- \rho_t) x_{t-1} + \rho_t \big[ - \frac{1}{\mu} A^\top f'(A x_{t-1} ) \big]\! \bigg),
$$
and corresponds to the choice $h'(x_{t-1}) = \mu x_{t-1}$ in the mirror descent recursion, which, when $x_{t-1}$ is on the boundary of $K$, is not the choice that we need for the equivalence in \mysec{cg}.

However, when $h$ is assumed to be differentiable on its closed domain $K$, then the bound of Prop.~\ref{prop:mirror} still holds because the optimality condition $h'(x_t) =  h'(x_{t-1}) -    \rho_t \big[
A^\top f'(A x_{t-1}) + h'(x_{t-1})  \big] $ may now be replaced by 
$( x - x_t)^\top \Big( h'(x_t) -  h'(x_{t-1}) +    \rho_t \big[
A^\top f'(A x_{t-1}) + h'(x_{t-1})  \big] \Big) \geqslant 0$ for all $x \in K$, which also allows to get to \eq{A} in the proof of Prop.~\ref{prop:mirror}.

\section{Conditional gradient method and extensions}
In this section, we first review the classical conditional gradient algorithm, which corresponds to the extra assumption that $f^\ast$ is linear in its domain.

\label{sec:cg}
\paragraph{Conditional gradient method}
 Given a maximization problem of the following form (i.e., where $f^\ast$ is linear on its domain, or equal to zero by a simple change of variable):
 $$
 \max_{y \in C} - h^\ast( - A^\top y ),
 $$
   the conditional gradient algorithm consists in the following iteration (note that below $A x_{t-1} = A (h^\ast)'(-A^\top y_{t-1})$ is the gradient of the objective function and that we are maximizing the first-order Taylor expansion to obtain a candidate $\bar{y}_{t-1}$ towards which we make a small step):
\BEAS
x_{t-1} & = &  \arg\min_{x \in \rb^p}   h(x) +  x^\top A^\top y_{t-1} 
\\
\bar{y}_{t-1} & \in &  \arg \max_{ y \in C} y^\top   A x_{t-1}     \\
y_t & = & (1-\rho_t) y_{t-1} + \rho_t \bar{y}_{t-1} . 
\EEAS

It corresponds to a linearization of $- h^\ast( - A^\top y )$ and its maximization over the bounded convex set $C$. As we show later, the choice of $\rho_t$ may be done in different ways, through a fixed step size of by (approximate) line search.

\paragraph{Generalization}
Following~\cite{bredies2008iterated}, the conditional gradient method can be generalized to problems of the form 
 $$
 \max_{y \in C} - h^\ast( - A^\top y ) - f^\ast(y),
 $$
 with the following iteration:
 \BEQ
 \label{eq:condgrad}
 \left\{
\begin{array}{ccl} 
x_{t-1} & = &  \arg\min_{x \in \rb^p}   h(x) +  x^\top A^\top y_{t-1} = (h^\ast)'(-A^\top y_{t-1})
\\[.25cm]
\bar{y}_{t-1} & \in &  \arg \max_{ y \in C} y^\top  A x_{t-1}  - f^\ast(y) \\[.25cm]
y_t & = & (1-\rho_t) y_{t-1} + \rho_t \bar{y}_{t-1}  .
\end{array} \right.
\EEQ

The previous algorithm may be interpreted as follows: (a) perform a first-order Taylor expansion of the smooth part $- h^\ast( - A^\top y ) $, while leaving the other part $ - f^\ast(y)$ intact, (b) minimize the approximation, and (c) perform a small step towards the maximizer. Note the similarity (and dissimilarity) with proximal methods which would add a proximal term proportional to $\| y - y_{t-1}\|^2$, leading to faster convergences, but with the extra requirement of solving the proximal step~\cite{Nesterov2007,Beck2009}.

Note that here $y_t$ may be expressed as a \emph{convex} combination of all $\bar{y}_{u-1}$, $u \in \{1,\dots,t\}$:
\BEAS
y_t & = & \sum_{u=1}^t  \bigg( \rho_u \prod_{s=u+1}^t ( 1- \rho_s) \bigg) \bar{y}_{u-1},
\EEAS
and that when we chose $\rho_t = 2/(t+1)$, it simplifies to:
\BEAS
y_t & = & \frac{2}{t(t+1)} \sum_{u=1}^t u \bar{y}_{u-1}.
\EEAS
When $h$ is essentially smooth (and thus $h^\ast$ is essentially strictly convex), it can be reformulated with $h'(x_{t}) = -  {A^\top y_{t}} $ as follows:
\BEAS
h'(x_t) & = &  (1-\rho_t) h'(x_{t-1})  -  {\rho_t}{ } A^\top   \arg \max_{ y \in C} \big\{ y^\top   A  x_{t-1}  - f^\ast(y) \big\}, \\
& = &  (1-\rho_t) h'(x_{t-1})  -  {\rho_t}{ } A^\top   f'(A x_{t-1}), 
\EEAS
which is exactly the mirror descent algorithm described in \eq{mirror}. This leads to the following proposition:

\begin{proposition}[Equivalence between mirror descent and generalized conditional gradient]
 \label{prop:equivalence}
  Assume that (a) $f$ is Lipschitz-continuous and finite on $\rb^p$, with $C$ the domain of $f^\ast$, (b) $h$ is $\mu$-strongly convex and essentially smooth. The mirror descent recursion in \eq{mirror}, started from $x_0 =(h^\ast)'(-A^\top y_{0})
$, is equivalent to the generalized conditional gradient recursion in \eq{condgrad}, started from $y_0 \in C$.
 \end{proposition}

When $h$ is not essentially smooth, then with a particular choice of subgradient (see end of \mysec{mirror}), the two algorithms are also equivalent. We now provide convergence proofs for the two versions (with adaptive and non-adaptive step sizes); similar rates may be obtained without the boundedness assumptions~\cite{bredies2008iterated}, but our results provide explicit constants and primal-dual guarantees.
We first have the following convergence proof for generalized conditional gradient with no line search (the proof of dual convergence uses standard arguments from~\cite{dem1967minimization,dunn1978conditional}, while the convergence of gaps is due to~\cite{jaggi} for the regular conditional gradient):

 \begin{proposition}[Convergence of extended conditional gradient - no line search]
 \label{prop:noline}
  Assume that (a) $f$ is Lipschitz-continuous and finite on $\rb^p$, with $C$ the domain of $f^\ast$, (b) $h$ is $\mu$-strongly convex. Consider $\rho_t = 2/(t+1)$ and $R^2 = \max_{y,y' \in C} \| A^\top ( y - y')\|^2$. Denoting by $y_\ast$   any maximizer of $g_{\rm dual}$ on $C$,  after $t$ iterations of the generalized conditional gradient recursion of \eq{condgrad}, we have:
 \BEAS
  g_{\rm dual} (y_\ast)  - g_{\rm dual}(y_t)  
  &  \leqslant  & \frac{2 R^2   }{  \mu (t+1)  } , \\
 \min_{ u \in \{0,\dots,t-1\}}   {\rm gap}(x_t,y_t) & \leqslant  &\frac{8 R^2   }{  \mu (t+1)  }.
 \EEAS
 \end{proposition}
\begin{proof}
We   have (using convexity of $f^\ast$ and $\big( \frac{1}{\mu}\big)$-smoothness of $h^\ast$):
 \BEAS
 & & g_{\rm dual}(y_t) \\
 & = & - h^\ast( - A^\top y_t ) - f^\ast(y_t)
 \\
 & \geqslant &
 \bigg[  - h^\ast( - A^\top y_{t-1} ) + ( y_t - y_{t-1} ) ^\top  A x_{t-1}  - \frac{R^2 \rho_t^2}{2\mu} \bigg] - \bigg[  (1-\rho_t) f^\ast(y_{t-1}) + \rho_t f^\ast(\bar{y}_{t-1}) \bigg]
\\
& = & - h^\ast( - A^\top y_{t-1} ) + \rho_t ( \bar{y}_{t-1} - y_{t-1} ) ^\top  A x_{t-1}  - \frac{R^2 \rho_t^2}{2\mu}  - (1-\rho_t) f^\ast(y_{t-1}) - \rho_t f^\ast(\bar{y}_{t-1})
\\
& = &g_{\rm dual}(y_{t-1}) + \rho_t ( \bar{y}_{t-1} - y_{t-1} ) ^\top  A x_{t-1}  - \frac{R^2 \rho_t^2}{2\mu}  + \rho_t f^\ast(y_{t-1}) -   \rho_t f^\ast(\bar{y}_{t-1})
\\
& = &g_{\rm dual}(y_{t-1}) - \frac{R^2 \rho_t^2}{2\mu}
    + \rho_t 
\bigg[
f^\ast(y_{t-1}) -     f^\ast(\bar{y}_{t-1})
+( \bar{y}_{t-1} - y_{t-1} ) ^\top  A x_{t-1}
\bigg]
\\
& = &g_{\rm dual}(y_{t-1}) - \frac{R^2 \rho_t^2}{2\mu}
    + \rho_t 
\bigg[
f^\ast(y_{t-1}) -  y_{t-1}  ^\top  A x_{t-1}
- ( f^\ast(\bar{y}_{t-1}) -  \bar{y}_{t-1}  ^\top  A x_{t-1}
 )\bigg].
 \EEAS
 Note that by definition of $\bar{y}_{t-1}$, we have (by equality in Fenchel-Young inequality)
 $$ - f^\ast(\bar{y}_{t-1}) + \bar{y}_{t-1}  ^\top  A x_{t-1}
 = f(Ax_{t-1}),
 $$
 and   $h^\ast(-A^\top y_{t-1} ) + h(x_{t-1}) + x_{t-1}^\top A^\top y_{t-1}=0 $, and thus
 $$
 f^\ast(y_{t-1}) -  y_{t-1}  ^\top  A x_{t-1}
- ( f^\ast(\bar{y}_{t-1}) -  \bar{y}_{t-1}  ^\top  A x_{t-1}
 )
  =  g_{\rm primal}(x_{t-1}) - g_{\rm dual}(y_{t-1}) =  { \rm gap}(x_{t-1},y_{t-1}).
 $$
 
 We thus obtain, for any $\rho_t \in [0,1]$:
 $$
 g_{\rm dual}(y_t) - g_{\rm dual}(y_\ast) \geqslant g_{\rm dual}(y_{t-1}) - g_{\rm dual}(y_\ast) + \rho_t { \rm gap}(x_{t-1},y_{t-1})
 - \frac{R^2 \rho_t^2}{2\mu},
 $$
 which is the classical equation from  the conditional gradient algorithm~\cite{dunn1978conditional,dunn1980convergence,jaggi},  which we can analyze through Lemma~\ref{lemma:1} (see end of this section), leading to the desired result.
  \end{proof}

The following proposition shows a result similar to the proposition above, but for the adaptive algorithm that considers optimizing the value $\rho_t$ at each iteration.

 \begin{proposition}[Convergence of extended conditional gradient - with line search]
 \label{prop:line}
 Assume that (a) $f$ is Lipschitz-continuous and finite on $\rb^p$, with $C$ the domain of $f^\ast$, (b) $h$ is $\mu$-strongly convex. Consider $\rho_t = 
 \min\{ \frac{\mu}{R^2} { \rm gap}(x_{t-1},y_{t-1}) , 1 \}$ and $R^2 = \max_{y,y' \in C} \| A^\top ( y - y')\|^2$. Denoting by $y_\ast$ any maximizer of $g_{\rm dual}$ on $C$,  after $t$ iterations of the generalized conditional gradient recursion of \eq{condgrad}, we have:
 \BEAS
  g_{\rm dual} (y_\ast)  - g_{\rm dual}(y_t)  
  &  \leqslant  & \frac{2 R^2   }{  \mu (t+3)  } , \\
 \min_{ u \in \{0,\dots,t-1\}}   {\rm gap}(x_t,y_t) & \leqslant  &\frac{2 R^2   }{  \mu (t+3)  }.
 \EEAS
 \end{proposition}

\begin{proof}
The proof is essentially the same as one from the previous proposition, with a different application of  Lemma~\ref{lemma:1} (see below).
\end{proof}

The following technical lemma is used in the previous proofs to obtain the various convergence rates.
 \begin{lemma}
 \label{lemma:1}
 Assume that we have three sequences $(u_t)_{t \geqslant 0}$,  $(v_t)_{t \geqslant 0}$, and  $(\rho_t)_{t \geqslant 0}$, and a positive constant $A$ such that
 \BEAS
 & & \forall t \geqslant 0, \ \rho_t \in [0,1] \\
 & & \forall t \geqslant 0, \  0 \leqslant u_t \leqslant v_t \\
& & \forall t \geqslant 1, \ u_t \leqslant u_{t-1} - \rho_t v_{t-1} + \frac{A}{2} \rho_t^2.
 \EEAS
 \BIT
  \item[--]
 If $\rho_t = 2/(t+1)$, then $ u_t \leqslant \frac{2A}{t+1}
 $ and for all $t \geqslant 1$, there exists at least one $k \in \{\lfloor t/2 \rfloor ,\dots,t\}$ such that $v_k \leqslant \frac{8A}{t+1} $.
 \item[--]
   If $\rho_t = \arg\min_{ \rho_t \in [0,1] } - \rho_t v_{t-1} + \frac{A}{2} \rho_t^2 = \min 
 \{ v_{t-1} / A
 ,1\}$, then $ u_t \leqslant \frac{2A}{t+3}
 $ and for all $t \geqslant 2 $, there exists at least one $k \in \{\lfloor t/2  \rfloor - 1 ,\dots,t\}$ such that $v_k \leqslant \frac{2A}{t+3} $.
 \EIT
 \end{lemma}
 \begin{proof}
 In the first   case (non-adaptive sequence $\rho_t$), we have $\rho_0=1$ and $u_t \leqslant ( 1 - \rho_t) u_{t-1} + \frac{A}{2} \rho_t^2$, leading to
 $$
 u_t \leqslant \frac{A}{2} \sum_{u=1}^t \prod_{s=u+1}^t ( 1- \rho_s) \rho_u^2.
 $$

 For $\rho_t = \frac{2}{t+1}$, this leads to
 $$
 u_t\leqslant \frac{A}{2} \sum_{u=1}^t \prod_{s=u+1}^t \frac{s-1}{s+1} =  \leqslant \frac{A}{2} \sum_{u=1}^t  \frac{u (u+1) }{t(t+1)}  \frac{4}{(u+1)^2}
 \leqslant \frac{2A}{t+1}.
 $$
 Moreover, for any $k < j$, by summing $u_t \leqslant u_{t-1} - \rho_t v_{t-1} + \frac{A}{2} \rho_t^2$ for $t \in \{k+1,\dots,j\}$, we get
 $$
 u_j  \leqslant u_k  - \sum_{t = k+1}^{j} \rho_t v_{t-1} + \frac{A}{2} \sum_{t=k+1}^j \rho_t^2.
 $$
 Thus, if we assume that  $v_{t-1} \geqslant \beta$ for all $t \in \{k+1,\dots, j\}$, then
 \BEAS
\beta  \sum_{t = k+1}^{j} \rho_t  \leqslant  \sum_{t = k+1}^{j} \rho_t v_{t-1} 
& \leqslant &  \frac{2A}{k+1}
 + 2A  \sum_{t=k+1}^j \frac{1}{(t+1)^2}
 \\
& \leqslant &  \frac{2A}{k+1}
 + 2A  \sum_{t=k+1}^j \frac{1}{t(t+1)}
\\
& = &  \frac{2A}{k+1}
 + 2A  \sum_{t=k+1}^j \big[ \frac{1}{t} - \frac{1}{t+1} \big] \leqslant   \frac{4A}{k+1}.
\EEAS 

 Moreover, $ \sum_{t = k+1}^{j} \rho_t
 =  2 \sum_{t = k+1}^{j} \frac{1}{t+1} \geqslant 2 \frac{ j  - k}{j+1}
 $. Thus
 $$
\beta  \leqslant \frac{ 2 A}{k+1} \frac{ j+1}{j - k}.
 $$
 Using $j = t+1$ and $k = \lfloor  t/ 2 \rfloor - 1$, we obtain that 
 $\beta \leqslant \frac{8A}{t+1}$ (this can be done by considering the two cases $t$ even and $t$ odd) and thus
 $\max_{ u \in \{\lfloor t  / 2 \rfloor, \dots, t \} } v_{u}
 \leqslant \frac{8A}{t+1}
 $.

 We now consider the line search case:
 \BIT
 \item[--] If $v_{t-1} \leqslant A$, then $\rho_t  = \frac{v_{t-1}}{A}$, and we obtain
 $u_t \leqslant u_{t-1}  - \frac{v_{t-1}^2}{2A}$.
 
 \item[--] If $v_{t-1} \geqslant A$, then $\rho_t  = 1$, and we obtain
 $u_t \leqslant u_{t-1}  - v_{t-1} + \frac{A}{2} 
 \leqslant  u_{t-1}  - \frac{v_{t-1}}{2}
  $.
 \EIT  
 Putting all this together, we get
 $ u_t \leqslant u_{t-1} - \frac{1}{2} \min \{ v_{t-1}, v_{t-1}^2 / A\}$.
 This implies that $(u_t)$ is a decreasing sequence. Moreover, $u_1 \leqslant \frac{A}{2}$ (because selecting $\rho_1=1$ leads to this value), thus, $u_1 \leqslant \min \{ u_0, A/2 \} \leqslant A$. We then obtain for all $t>1$,
 $u_t \leqslant u_{t-1} - \frac{1}{2A} u_{t-1}^2$. From which we deduce,
 $u_{t-1}^{-1} \leqslant u_t^{-1} -  \frac{1}{2A}$. We can now sum these inequalities to get
 $u_1^{-1} \leqslant u_t^{-1} - \frac{t-1}{2A}$, that is,
 $$ u_t \leqslant \frac{1}{ u_1^{-1} + \frac{t-1}{2A}}
 \leqslant \frac{1}{ \max \{ u_0^{-1}, 2/ A \} + \frac{t-1}{2A}}
 \leqslant \frac{2A}{t+3} .$$
 
  Moreover, if we assume that all $v_{t-1} \geqslant \beta$ for $t \in \{k+1,\dots, j\}$, following the same reasoning as above,
  and using the inequality  $ u_t \leqslant u_{t-1} - \frac{1}{2} \min \{ v_{t-1}, v_{t-1}^2 / A\}$, we obtain
  $$
  \min \{ \beta , \beta^2 / A \} ( j -k ) \leqslant \frac{A}{k+3}.
  $$
   Using $j = t+1$ and $k = \lfloor  t/ 2 \rfloor - 1$, we have
   $(k+3) (j-k) > \frac{1}{4}(t+3)^2$ (which can be checked by considering the two cases $t$ even and $t$ odd). Thus, we must have $\beta \leqslant A$ (otherwise we obtain $\beta 
   \leqslant 4 A / ( t+3)^2$, which is a contradiction with $\beta \geqslant A$), and thus
   $\beta^2 \leqslant  4 A^2 / ( t+3)^2$, which leads to the desired result.
  \end{proof}

 \section{Discussion}
 
 The equivalence shown in Prop.~\ref{prop:equivalence} has several interesting consequences and leads to several additional related questions:
\BIT
 
\item[--] \textbf{Primal-dual guarantees}: Having a primal-dual interpretation directly leads to primal-dual certificates, with a gap that converges at the same rate  proportional to $\frac{R^2}{\mu t}$ (see~\cite{jaggi,lacoste2012stochastic} for similar results for the regular conditional gradient method). These certificates may first be taken to be the pair $(x_t,y_t)$, in which case, we have shown that after $t$ iterations, at least one of the previous iterates has the guarantee.  
Alternatively, for the fixed step-size $\rho_t = \frac{2}{t+1}$, we can use the same dual candidate $y_t = \frac{2}{t(t+1)}\sum_{u=1}^t u  \bar{y}_{u-1} $  (which can thus also be expressed as an average of subgradients) and averaged primal iterate $\frac{2}{t(t+1)}\sum_{u=1}^t u  x_{u-1} $. Thus, the two weighted averages of subgradients lead to primal-dual certificates.

\item[--] \textbf{Line-search for mirror descent}: Prop.~\ref{prop:line} provides a form of line search for mirror descent (i.e., an adaptive step size). Note the similarity with Polyak's rule which applies to the non-strongly convex case (see, e.g.,~\cite{bertsekas}).  

\item[--] \textbf{Absence of logarithmic terms}: Note that we have considered a step-size of $\frac{2}{t+1}$, which avoids a logarithmic term of the form $\log t$ in all bounds (which would be the case for $\rho_t =\frac{1}{t}$). This also applies to the stochastic case~\cite{stochastic}.

\item[--] \textbf{Properties of iterates}: While we have focused primarily on the convergence rates of the iterates and their objective values, recent work has shown that the iterates themselves could have interesting distributional properties~\cite{welling2009herding,bach2012equivalence}, which would be worth further investigation.

 \item[--] \textbf{Stochastic approximation and online learning}: There are potentially other exchanges between primal/dual formulations, in particular in the stochastic setting (see, e.g.,~\cite{lacoste2012stochastic}).
 
 \item[--] \textbf{Simplicial methods and cutting-planes}: The duality between subgradient and conditional gradient may be extended to algorithms with iterations that are more expensive. For example, simplicial methods  in the dual are equivalent to cutting-planes methods in the primal~(see, e.g., \cite{bertsekas2011unifying,lacoste2012stochastic} and~\cite[Chapter 7]{bach2011learning}).

 \EIT

\subsection*{Acknowledgements}
This work was partially supported by  the European Research Council (SIERRA Project). The author would like to
thank Simon Lacoste-Julien, Martin Jaggi, Mark Schmidt and Zaid Harchaoui for discussions related to convex optimization and conditional gradient algorithms.

 \bibliography{SG_CG}
\bibliographystyle{plain}

\newpage

\appendix

\section{Non-strongly convex case}
\label{app:nonstrongly} 
In this appendix, we consider the situation where the primal optimization problem is just convex. That is, we assume that we are given (a) a Lipschitz-continuous function $f$ defined on $\rb^n$ (with $C$ the domain of its Fenchel-conjugate), (b) a lower-semicontinuous and $1$-strongly convex function $h$ with \emph{compact} domain $K$, which is differentiable on ${\rm int}(K)$ and  such that for all $(x_1,x_2) \in  K \times {\rm int}(K)$, $D(x_1,x_2) \leqslant  \delta^2 $, and (c) a matrix $ A \in \rb^{n \times p}$. We consider the problem,
$$
\min_{x \in K} f(A x),
$$
and we let $x_\ast \in K$ denote any minimizer. We have the following Fenchel duality relationship:
\BEAS
\min_{x \in K} f(A x)
& = & \min_{x \in K} \max_{y \in C} y^\top A x - f^\ast(y) \\
& = &  \max_{y \in C} - \sigma(-A^\top y) - f^\ast(y),
\EEAS
where $\sigma: \rb^p \to \rb$ is the support function of $K$ defined as $\displaystyle \sigma(z) = \max_{x \in K} z^\top x$. We consider the mirror descent recursion~\cite{nemirovsky1983problem,beck2003mirror}, started from $x_0 \in K$ and for which, for $t \geqslant 1$,
$$
z_t \in \arg\min_{x \in K} \ \frac{1}{\rho_t} D(x,x_{t-1}) +  (x-x_{t-1})^\top A^\top y_{t-1},
$$
where $y_{t-1}$  is a subgradient of $f$ at $ Ax_{t-1}$.
Our goal is to show that the average iterate $\bar{x}_t = \frac{1}{t} \sum_{u=0}^{t-1} x_u$ and the average dual candidate
$\bar{y}_t = \frac{1}{t} \sum_{u=0}^{t-1} y_u$ are such that
$${\rm gap}(\bar{x}_t,\bar{y}_t) =  f(A \bar{x}_t) + \sigma(-A^\top \bar{y}_t) + f^\ast(\bar{y}_t)$$ tends to zero at an appropriate rate. Similar results hold for certain cases of subgradient descent~\cite{nedic2009} and we show that they hold more generally. 

Let $x \in K$. We have (using a similar reasoning than~\cite{beck2003mirror}) and using the optimality condition $(x-x_t)^\top \big[ h'(x_t) -h'(x_{t-1}) + \rho_t A^\top y_{t-1}\big]\geqslant 0$ for any $x \in K$:
\BEAS
\!\!\!\! D(x,x_{t}) - D(x,x_{t-1})
& \!\! =\!\! & - h(x_{t}) - h'(x_{t})^\top(x-x_{t}) + h(x_{t-1}) + h'(x_{t-1})^\top(x-x_{t-1}) \\
& \!\! \leqslant \!\! &      \rho_t ( x - x_{t-1})^\top A^\top y_{t-1} + h(x_{t-1}) - h(x_{t}) - h'(x_{t})^\top (x_{t-1} - x_{t})
\\
& \! \! \leqslant \!\! & \rho_t ( x - x_{t-1})^\top A^\top y_{t-1} + \big[ h'(x_{t-1})- h'(x_{t}) \big]^\top (x_{t-1} - x_{t}) - \frac{1}{2} \| x_t -x_{t-1}\|^2
\\
& & \hspace*{4cm} \mbox{ using the 1-strong convexity of } h, \\
& \! \! = \!\! & \rho_t ( x - x_{t-1})^\top A^\top y_{t-1} + \big[\rho_t A^\top y_{t-1} \big]^\top (x_{t-1} - x_{t}) - \frac{1}{2} \| x_t -x_{t-1}\|^2 \\
& \! \! = \!\! & \rho_t ( x - x_{t-1})^\top A^\top y_{t-1} +   \frac{1}{2}\| \rho_t A^\top y_{t-1} \|^2
- \frac{1}{2} \| x_t - x_{t-1} + \rho_t A^\top y_{t-1} \|^2\\
& \! \! \leqslant \!\! & \rho_t ( x - x_{t-1})^\top A^\top y_{t-1} + \frac{1}{2}\| \rho_t A^\top y_{t-1} \|^2. \EEAS
This leads to
$$  (  x_{t-1} - x)^\top A^\top y_{t-1} \leqslant \frac{1}{\rho_t} \big[ D(x,x_{t-1}) - D(x,x_{t})  \big] + \frac{\rho_t R^2}{2},$$
where $R^2 = \max_{y \in C} \|A^\top y\|^2$ (note the slightly different definition than in \mysec{cg}). By summing from $u=1$ to $t$, we obtain
\BEAS
 \sum_{u=1}^{t} (  x_{u-1} - x)^\top A^\top y_{u-1} &  \leqslant &    \sum_{u=1}^{t} \frac{1}{\rho_u} \big[ D(x,x_{u-1}) - D(x,x_{u})  \big] +  \sum_{u=1}^{t} \frac{\rho_u R^2}{2}.
\EEAS
Assuming that $(\rho_t)$ is a decreasing sequence, we get by integration by parts:
\BEAS
 \sum_{u=1}^{t} (  x_{u-1} - x)^\top A^\top y_{u-1} &  \leqslant &    \sum_{u=1}^{t-1}  D(x,x_{u}) \Big(  \frac{1}{\rho_{u+1}}
 - \frac{1}{\rho_{u}} \Big)  + \frac{D(x,x_{0}) }{\rho_1}
 - \frac{D(x,x_{t}) }{\rho_t} +  \sum_{u=1}^{t} \frac{\rho_u R^2}{2}\\
  &  \leqslant &    \sum_{u=1}^{t-1}    \delta^2  \Big(  \frac{1}{\rho_{u+1}}
 - \frac{1}{\rho_{u}} \Big)  + \frac{\delta^2 }{ \rho_1}
 +  \sum_{u=1}^{t} \frac{\rho_u R^2}{2}\\
   &  =  &      \frac{\delta^2 }{  \rho_t}
 +  \frac{ R^2}{2} \sum_{u=1}^{t}\rho_u  .
\EEAS
We may now compute an upper-bound on the gap as follows:
\BEAS
{\rm gap}(\bar{x}_t,\bar{y}_t)
& = &   f(A \bar{x}_t) + \sigma(-A^\top \bar{y}_t) + f^\ast(\bar{y}_t) \\
& \leqslant & \frac{1}{t} \sum_{u=0}^{t-1} f(Ax_u) + 
\frac{1}{t}\sum_{u=0}^{t-1}  f^\ast(y_u)  + \sigma(-A^\top \bar{y}_t) \\
& = &  \frac{1}{t} \sum_{u=0}^{t-1} y_u^\top Ax_u  + \sigma(-A^\top \bar{y}_t) \\
& = &  \frac{1}{t} \sum_{u=0}^{t-1} y_u^\top A( x_u-x)  + \sigma(-A^\top \bar{y}_t) +x^\top A^\top \bar{y}_t.
\EEAS
Using the bound above and minimizing with respect to $x \in K$, we obtain
$$
{\rm gap}(\bar{x}_t,\bar{y}_t) \leqslant   \frac{\delta^2 }{  t \rho_t}
 +  \frac{ R^2}{2t} \sum_{u=1}^{t}\rho_u.
$$
With $\displaystyle \rho_t = \frac{ {\delta}}{R \sqrt{t}}$, we obtain a gap less than
$$ \frac{ R  {\delta}}{  \sqrt{t}} + \frac{ R  {\delta}}{ t} \sum_{k=1}^t \big[ \sqrt{k} - \sqrt{k-1} \big]
\leqslant   \frac{  2  R  {\delta}}{  \sqrt{t}}.  $$

\end{document}